\documentclass[final,onefignum,onetabnum]{siamonline250211}

\usepackage{amsmath,amsfonts,braket}
\usepackage{amsopn} 

\usepackage{graphicx,epstopdf}
\usepackage{subcaption}

\usepackage{algpseudocode}
\Crefname{ALC@unique}{Line}{Lines}

\newsiamthm{claim}{Claim}
\newsiamremark{remark}{Remark}
\newsiamremark{hypothesis}{Hypothesis}
\crefname{hypothesis}{Hypothesis}{Hypotheses}

\usepackage{xspace}
\usepackage{bold-extra}
\usepackage[most]{tcolorbox}
\usepackage{pgfplots}

\colorlet{texcscolor}{blue!50!black}
\colorlet{texemcolor}{red!70!black}
\colorlet{texpreamble}{red!70!black}
\colorlet{codebackground}{black!25!white!25}

\lstdefinestyle{siamlatex}{%
  style=tcblatex,
  texcsstyle=*\color{texcscolor},
  texcsstyle=[2]\color{texemcolor},
  keywordstyle=[2]\color{texemcolor},
  moretexcs={cref,Cref,maketitle,mathcal,text,headers,email,url},
}

\tcbset{%
  colframe=black!75!white!75,
  coltitle=white,
  colback=codebackground,
  colbacklower=white,
  fonttitle=\bfseries,
  arc=0pt,outer arc=0pt,
  top=1pt,bottom=1pt,left=1mm,right=1mm,middle=1mm,boxsep=1mm,
  leftrule=0.3mm,rightrule=0.3mm,toprule=0.3mm,bottomrule=0.3mm,
  listing options={style=siamlatex}
}

\newtcblisting[use counter=example]{example}[2][]{%
  title={Example~\thetcbcounter: #2},#1}

\newtcbinputlisting[use counter=example]{\examplefile}[3][]{%
  title={Example~\thetcbcounter: #2},listing file={#3},#1}

\DeclareTotalTCBox{\code}{ v O{} }
{fontupper=\ttfamily\color{black},
  nobeforeafter,
  tcbox raise base,
  colback=codebackground,colframe=white,
  top=0pt,bottom=0pt,left=0mm,right=0mm,
  leftrule=0pt,rightrule=0pt,toprule=0mm,bottomrule=0mm,
  boxsep=0.5mm,
  #2}{#1}

\patchcmd\newpage{\vfil}{}{}{}
\def\E{\mathbb{E}}
\def\Dist{\mathcal{D}}
\def\Hyp{\mathcal{H}}
\def\Man{\mathcal{M}}
\def\R{\mathbb{R}}
\DeclareMathOperator*{\argmin}{arg\,min}

\begin{tcbverbatimwrite}{tmp_\jobname_header.tex}
\title{Continuous Symmetry Discovery and Enforcement Using Infinitesimal Generators of Multi-parameter Group Actions\thanks{This work is an extension of \textit{Symmetry Discovery Beyond Affine Transformations} \cite{Shaw1}.
\funding{This research was supported in part by the NSF under Grants 2212325 [K.M.], CIF-2212327 [A.M.],
and CIF-2338855 [A.M.].}}}

\author{Ben Shaw\thanks{Utah State University (\email{ben.shaw@usu.edu}).}
\and Sasidhar Kunapuli\thanks{UC Berkeley (\email{sasidhar.kunapuli@gmail.com}).}
\and Abram Magner\thanks{University at Albany (\email{amagner@albany.edu}).}
\and Kevin R. Moon\thanks{Utah State University (\email{kevin.moon@usu.edu}).}}

\headers{Symmetry Discovery, Enforcement Using Generators}{Shaw et al.}
\end{tcbverbatimwrite}
\input{tmp_\jobname_header.tex}

\ifpdf
\hypersetup{ pdftitle={Continuous Symmetry Discovery and Enforcement Using Infinitesimal Generators} }
\fi

\begin{document}
\maketitle

\begin{tcbverbatimwrite}{tmp_\jobname_abstract.tex}
\begin{abstract}

Symmetry-informed machine learning can exhibit advantages over machine learning which fails to account for symmetry. In the context of continuous symmetry detection, current state of the art experiments are largely limited to detecting affine transformations. Herein, we outline a computationally efficient framework for discovering infinitesimal generators of multi-parameter group actions which are not generally affine transformations. This framework accommodates the automatic discovery of the number of linearly independent infinitesimal generators. We build upon recent work in continuous symmetry discovery by extending to neural networks and by restricting the symmetry search space to infinitesimal isometries. We also introduce symmetry enforcement of smooth models using vector field regularization, thereby improving model generalization. The notion of vector field similarity is also generalized for non-Euclidean Riemannian metric tensors.


\end{abstract}

\begin{keywords}
Symmetry Discovery, Machine Learning, Symmetry Enforcement, Killing Vectors, Lie Derivative
\end{keywords}

\begin{MSCcodes}
62-07, 68T01
\end{MSCcodes}
\end{tcbverbatimwrite}
\input{tmp_\jobname_abstract.tex}

\section{Introduction}
\label{sec:intro}


Approaching various machine learning tasks with prior knowledge, commonly in the form of symmetry present in datasets and tasks, has been shown to improve performance and/or computational efficiency \cite{Lyle1,Bergman1,Craven1,tahmasebi2023exact}. While research in symmetry detection in data has enjoyed recent success \cite{LieGAN}, the current state of the art leaves room for improvement, particularly with respect to the detection of continuous symmetry \cite{Zhou1}. In this work, we seek to address two issues at hand: primarily, the inability of current methods to detect continuous symmetry beyond affine symmetries; secondly, the apparent increase of difficulty in identifying continuous symmetries when compared with discrete symmetries.

Our methodology of symmetry discovery is summarized as follows. First, given a dataset, certain quantities of interest, hereafter referred to as \textit{machine learning functions}, are identified, such as an underlying probability distribution which generates the dataset, a classification or regression function, or a dimensionality reduction function. Second, we estimate tangent vector fields which annihilate the machine learning functions. The estimated vector fields themselves are generators of continuous symmetries. Thus, estimating the vector fields provides an indirect estimate of the continuous symmetries. 


Vector fields have been used to describe symmetry in the context of machine learning before \cite{Lconv}. However, our use of vector fields differs from previous approaches for two reasons. First, we fully exploit the vector fields as objects which operate on machine learning functions directly. Second, we estimate vector fields associated with continuous symmetries beyond affine transformations.

We also extend the previous work on this subject in nontrivial ways \cite{Shaw1}. While the previous work introduced the notion of symmetry enforcement via the estimation of functions which are invariant with respect to a fixed collection of vector fields, we explore the method of regularization for symmetry enforcement. This method is critical for tasks for which ``strict" invariance is not appropriate, or for tasks for which a suitable collection of invariant functions cannot be identified or estimated. A more mathematically rigorous treatment of the estimation of invariant functions is also provided. We also provide a brief overview of infinitesimal generators for multi-parameter group actions in order to highlight that our methodology is suited for the multi-parameter case. We provide more detail and experiments regarding the restriction to special types of vector fields such as Killing vectors: this restriction, as noted in the previous work, can have a particularly positive impact on the estimation of vector fields whose flows are not affine transformations. We generalize the notion of comparing discovered/estimated vector fields to ground-truth vector fields and propose a means to obtain ``symmetry scores" to assess the extent to which models are invariant with respect to a vector field. We provide a discussion on the robustness of our symmetry estimation method, and we present all-new experimental results.

This paper is organized as follows. In the next section, we briefly discuss previous work in symmetry detection and enforcement. In Section \ref{sec:background}, we discuss vector fields, flows, infinitesimal generators of multi-parameter group actions, and isometries, the mathematical principles that are foundational to our methods. In Section \ref{sec:methods}, we describe our methods, and in Section \ref{sec:experiments}, we present our experimental results. We then offer concluding remarks and outline future work in the area.

\section{Related Work}
\label{sec:liturature}

\subsection{Symmetry Discovery}


Early work on symmetry detection in machine learning focused primarily on detecting symmetry in image and video data \cite{RaoRuderman, Dickstein1}, where symmetries described by straight-line and rotational transformations were discovered.
Other work has made strides in symmetry discovery by restricting the types of symmetries being sought. In one case, detection was limited to compact Abelian Lie groups \cite{Cohen1}, since such groups exhibit mathematically convenient properties, and used for the purpose of learning disentangled representations. Another case uses meta-learning to discover any \textit{finite} symmetry group \cite{Zhou1}. Finite groups have also been used in symmetry discovery in representation learning \cite{Anselmi1}. In physics-based applications of machine learning, a method has been constructed which is able to discover any classical Lie group symmetry \cite{Forestano1}. Symmetry discovery of shapes has also been explored \cite{Je1}.

Other work has focused on detecting affine transformation symmetries and encoding the discovered symmetries automatically into a model architecture. Three such methods identify Lie algebra generators to describe the symmetries, as we do herein. For example, \textit{augerino} \cite{Augerino} attempts to learn a distribution over augmentations, subsequently training a model with augmented data. The \textit{Lie algebra convolutional network} \cite{Lconv}, which generalizes Convolutional Neural Networks in the presence of affine symmetries, uses infinitesimal generators represented as vector fields to describe the symmetry. SymmetryGAN \cite{Desai1} has also been used to detect rotational symmetry \cite{LieGAN}.

Another notable contribution to efforts to detect symmetries of data is \textit{LieGAN}. LieGAN is a generative-adversarial network intended to return infinitesimal generators of the continuous symmetry group of a given dataset \cite{LieGAN}. LieGAN has been shown to detect continuous affine symmetries, including transformations from the Lorentz group. It has also been shown to identify discrete symmetries such as rotations by a fixed angle.

While most continuous symmetry detection methods attempt to discover symmetries which are affine transformations, the representation of infinitesimal generators using vector fields has led to the discovery of continuous symmetries which are not affine \cite{Ko1, Shaw1}. We note that our work here is an extension of \cite{Shaw1}. In \cite{Ko1}, the domains of image data and partial differential equations are examined in particular, and the symmetries are expressed as flows of vector field infinitesimal generators.
However, the approach taken therein differs fundamentally from our approach. In both approaches, the infinitesimal generators are iteratively updated; however, in \cite{Ko1}, training data is explicitly transformed by the flow of the vector field, which flow is estimated using a neural network ODE solver. In our work, the infinitesimal generators are taken to act directly upon functions of interest, eliminating the need for data augmentation.  Furthermore, the performance of our approach does not depend on the selection of an ODE solver, in contrast to \cite{Ko1}.

Continuous symmetry detection differs from discrete symmetry detection \cite{Zhou1} since the condition $f \circ S = f$ must hold for all values of the continuous parameter of $S$. This is corroborated by the increasingly complex methods used to calculate even simple symmetries such as planar rotations \cite{Augerino, Lconv, LieGAN}. Some methods introduce discretization, where multiple parameter values are chosen and evaluated.  LieGAN does this by generating various transformations from the same infinitesimal generator \cite{LieGAN}. However, introducing discretization increases the complexity of continuous symmetry detection. 
A vector field approach addresses the issue of discretization. Our vector field-based method reduces the required model complexity of continuous symmetry detection while offering means to detect symmetries beyond affine transformations.

\subsection{Symmetry Enforcement}

We next review methods for enforcing symmetries in machine learning models.  Some methods seek to enforce symmetry by augmenting the training dataset according to known symmetries \cite{Bergman1}. This is common practice in image classification tasks, where augmentations include rotated copies of images in the training set. \textit{Augerino} attempts to enforce symmetry using augmented data, although the symmetries are discovered from the data rather than given \textit{a priori}. Another established method of enforcing symmetry is feature averaging, which, in contrast to data augmentation, guarantees symmetry enforcement~\cite{Lyle1}.
Other sources construct target symmetry-invariant and -equivariant models using infinitesimal generators \cite{Lconv, LieGAN}. Some previous work addresses specific cases including the special case of compact groups \cite{Bloem-Reddy1} and  equivariant CNNs on homegeneous spaces \cite{Cohen2}. Other works prove universal approximation results for invariant architectures \cite{Maron1, Keriven1, Yarotsky1}.

Physics-Informed Neural Networks (PINNs) \cite{Raissi1} may be seen as enforcing a type of symmetry. Here, model training is regularized using differential constraints which represent the governing equations for a physical system. Our regularization method of symmetry enforcement 
adopts a similar approach, though the differential constraints obtained using infinitesimal generators do not generally have the interpretation of defining governing equations for a physical system. Notably, the topic of symmetry enforcement via regularization has been independently proposed, approximately simultaneously with our own efforts \cite{Otto1}.

\section{Background}
\label{sec:background}

In this section, we give an overview of vector fields and flows and their connection to symmetry. We also discuss Killing vectors and isometries. 
Throughout, we assume a dataset $\mathcal{D}=\{x_1,\dots,x_r\}$ with $x_i\in\mathbb{R}^n$. Various machine learning functions of interest may be defined for a given dataset, such as an underlying probability distribution or class probabilities. The object of symmetry detection in data is to discover functions $S:\mathbb{R}^n \to \mathbb{R}^n$ that preserve a machine learning function of interest: that is, for a machine learning function $f$, $f \circ S = f$. We consider a continuous symmetry of a particular machine learning function to be a transformation $S$ which is continuously parameterized, such as a rotation in a plane by an angle $\theta$, with $\theta$ being the continuous parameter. We deal exclusively with continuous symmetries.

\subsection{Vector Fields and Flows}

We now provide some background on vector fields and their associated flows. We refer the reader to literature on the subject for additional information \cite{Lee1}. Suppose that $X$ is a smooth\footnote{That is, a vector field with $\mathcal{C}^{\infty}$ coefficient functions.} (tangent) vector field on $\mathbb{R}^n$:
\begin{equation}
    X = \alpha^i \partial_{x^i} := \sum_{i=1}^n \alpha^i \partial_{x^{i}},
\end{equation}
where $\alpha^i:\mathbb{R}^n \to \mathbb{R}$ for $i \in [1,n]$, and where $\{x^i\}_{i=1}^n$ are coordinates on $\mathbb{R}^n$.  $X$ assigns a tangent vector at each point and can also be viewed as a function on the set of smooth, real-valued functions. E.g. if $f:\mathbb{R}^n \to \mathbb{R}$ is smooth,
\begin{equation}{\label{vf}}
    X(f) = \sum_{i=1}^n \alpha^i \frac{\partial f}{\partial{x^{i}}}.
\end{equation}
For example, for $n=2$, if $f(x,y)=xy$ and $X=y\partial_x$, then $X(f) = y^2$. $X$ is also a \textit{derivation} on the set of smooth functions on $\mathbb{R}^n$: that is, 
as a mapping from $C^{\infty}(\mathbb{R}^n)$ to itself, it behaves like a derivative in the sense that it is linear and satisfies the product rule.
A flow on $\mathbb{R}^n$ is a smooth function $\Psi: \mathbb{R} \times \mathbb{R}^n \to \mathbb{R}^n$ which satisfies
\begin{equation}
    \Psi(0,p) = p, \qquad \Psi(s,\Psi(t,p)) = \Psi(s+t,p)
\end{equation}
for all $s,t \in \mathbb{R}$ and for all $p \in \mathbb{R}^n$. A flow is a 1-parameter group action. An example of a flow $\Psi: \mathbb{R} \times \mathbb{R}^2 \to \mathbb{R}^2$ is
\begin{equation}{\label{flowEx}}
    \Psi(t,(x,y)) = (x\cos(t) - y\sin(t), x\sin(t)+y\cos(t)),
\end{equation}
with $t$ being the continuous parameter known as the flow parameter. This flow rotates a point $(x,y)$ about the origin by $t$ radians. 

For a given flow $\Psi$, one may define a (unique) vector field $X$ as given in Equation \ref{vf}, where each function $\alpha^i$ is defined as
$ 
    \alpha^i = \left( \frac{\partial \Psi}{\partial t} \right) \bigg|_{t=0}.
$ 
Such a vector field is called the infinitesimal generator of the flow $\Psi$. For example, the infinitesimal generator of the flow given in Equation \ref{flowEx} is $-y\partial_x + x\partial_y$.

Conversely, given a vector field $X$ as in Equation \ref{vf}, one may define a corresponding flow as follows. Consider the following system of differential equations:
\begin{equation}{\label{vfToFlow}}
    \frac{d x^i}{dt} = \alpha^i, \qquad x^i(0) = x^i_{0}.
\end{equation}
Suppose that a solution $\mathbf{x}(t)$ to Equation \ref{vfToFlow} exists for all $t \in \mathbb{R}$ and for all initial conditions $\mathbf{x}_0 \in \mathbb{R}^n$. Then the function $\Psi:\mathbb{R} \times \mathbb{R}^n \to \mathbb{R}^n$ given by
$ 
    \Psi(t,\mathbf{x}_0) = \mathbf{x}(t)
$ 
is a flow. The infinitesimal generator corresponding to $\Psi$ is $X$. For example, to calculate the flow of $-y\partial_x + x\partial_y$, we solve
\begin{equation}
    \dot{x} = -y, \quad \dot{y} = x, \qquad x(0) = x_0, \quad y(0) = y_0
\end{equation}
and obtain the flow $\Psi(t,(x_0,y_0))$ defined by Equation \ref{flowEx}. It is generally easier to obtain the infinitesimal generator of a flow than to obtain the flow of an infinitesimal generator.


\subsection{Vector Fields and Symmetry}

We can now connect vector fields and flows with symmetry. A smooth function $f:\mathbb{R}^n \to \mathbb{R}$ is said to be $X$-invariant if $X(f) = 0$ identically for a smooth vector field $X$. The function $f$ is $\Psi$-invariant if, for all $t \in \mathbb{R}$, $f = f(\Psi(t,\cdot))$ for a flow $\Psi$. If $X$ is the infinitesimal generator of $\Psi$, $f$ is $\Psi$-invariant if and only if $f$ is $X$-invariant. If the function $f$ is a machine learning function for a given data set, our strategy is to identify vector fields $X_i$ for which $f$ is invariant. Each vector field $X_i$ for which $f$ is invariant is associated, explicitly or implicitly, with a flow $\Psi_i$, each of which is a 1-parameter group action, the collection of which generate the multi-parameter group action under which $f$ is invariant. By identifying continuous symmetries by approximately solving $X(f)=0$ for parametric $X$, 
we boil continuous symmetry discovery down to a tractable optimization problem.

The price of this indirect characterization of continuous symmetry is that the group action is only implicitly represented by its infinitesimal generators.
The problem of building machine learning models that are invariant to symmetries corresponding to given infinitesimal generators is called \emph{symmetry enforcement},
and we describe our results on this in Section~\ref{sec:enforcement}.

\subsection{Infinitesimal Generators of Multi-Parameter Group Actions}

We emphasize that our methods also apply to symmetries given by multi-parameter group actions, given recent work seeking to compare our method \cite{hu1}. 

Let $G \subseteq \mathbb{R}^s$ be a group, and suppose $G$ acts on $\mathbb{R}^n$: that is, for $g_1, g_2 \in G$ and for $x \in \mathbb{R}^n$, there is a function $\Psi: G \times \mathbb{R}^n \to \mathbb{R}^n$ such that (assuming the group operation is vector addition)
\begin{equation}{\label{multigroup}}
    \Psi(\mathbf{0},x) = x, \qquad \Psi(g_2,\Psi(g_1,x)) = \Psi(g_1+g_2,x).
\end{equation}
The use of the symbol $\Psi$ to denote a multi-parameter group action is not accident, as a flow is a 1-parameter group action. Let $\{v_i\}_{i=1}^s$ be a basis for the tangent space of $G$ at $\mathbf{0}$, the group identity element. Lastly, let $\sigma$ be a curve in $G$ for which $\sigma(t_0) = \mathbf{0}$ and $\dot{\sigma}(t_0)=v_i$ for $t_0 \in \mathbb{R}$. The infinitesimal generator $X_i$ corresponding to $v_i$ is given by
\begin{equation}{\label{multigroupX}}
    X_i = \left(\dfrac{d}{dt} \Psi(\sigma(t),x)\right) \bigg|_{t = t_0}.
\end{equation}
For example, consider the group $G=\mathbb{R}^3$ acting on $\mathbb{R}^2$ via
\begin{equation*}
    \Psi((a,b,\theta),(x,y)) = \left( x\cos(\theta) - y\sin(\theta) + a, x\sin(\theta) + y\cos(\theta) + b \right).
\end{equation*}
Given the following three curves,
$ 
    \sigma_a(t) = (t,0,0), \qquad \sigma_b(t) = (0,t,0), \qquad \sigma_{\theta}(t) = (0,0,t),
$ 
we find that
$ 
    X_a = \dfrac{d}{dt} \left( t,0 \right) |_{t=0} = \partial_x, \qquad X_b = \dfrac{d}{dt} \left( 0,t \right) |_{t=0} = \partial_y,
 $ \\ 
$ 
    X_{\theta} = \dfrac{d}{dt} \left( x\cos(\theta) - y\sin(\theta), x\sin(\theta) + y\cos(\theta) \right) |_{t=0} = -y\partial_x + x\partial_y.
$ 
For each of these vector fields, a corresponding flow can be computed, which flows we call $\Psi_a$, $\Psi_b$, and $\Psi_{\theta}$, respectively. In terms of the original parameters, these flows are given as
$ 
    \Psi_a(a,(x,y)) = (x+a,y), 
    \Psi_b(b,(x,y)) = (x,y+b), $ and
    $\Psi_{\theta} = \left( x\cos(\theta) - y\sin(\theta), x\sin(\theta) + y\cos(\theta) \right)$.
While each of these flows are, individually, 1-parameter group actions, it is clear that the infinitesimal generators $X_a$, $X_b$, and $X_{\theta}$ are the infinitesimal generators for the multi-parameter group action given in Equation (\ref{multigroup}). Thus, discovering vector field infinitesimal generators which annihilate a fixed (smooth) function applies to multi-parameter group actions and not solely to 1-parameter group actions.

\subsection{Killing Vectors and Isometries}




Suppose that a manifold $\mathcal{M}$ is equipped with a Riemannian metric tensor $g$. While the metric tensor is fundamentally a family of inner products defined on the tangent bundle $T\mathcal{M}$ of $\mathcal{M}$, it also induces a distance measure $d$ on $\mathcal{M}$, giving $\mathcal{M}$ the structure of a metric space. A Killing vector $X \in T\mathcal{M}$ of $g$ can be defined as a vector field which annihilates $g$ by means of the Lie derivative:
\begin{equation}{\label{kvEq}}
    \mathcal{L}_{X}g = 0.
\end{equation}
If a Killing vector $X$ is the infinitesimal generator of a flow $\Psi$, then for two points $p,q \in \mathcal{M}$,
$ 
    d(\Psi(t,p), \Psi(t,q)) = d(p,q)
$ 
for any value of the flow parameter $t$, so that $\Psi$ is an isometry of the distance measure $d$ induced by the metric tensor $g$. We note that the set of Killing vectors of a metric $g$ form a vector space as well as a Lie algebra (using the Lie bracket for vector fields).
Killing vectors of a particular metric $g$ can be calculated by solving Equation~\ref{kvEq}, which defines a system of overdetermined, linear, homogeneous, partial differential equations known as the Killing equations \cite{Gover1}, in general numerically. 
\footnote{Computational methods have been developed which reduce the complexity of the Killing equations \cite{Shaw2}. Herein, such methods are not needed, since the Killing equations for our experiments are sufficiently simple.} 

In the context of machine learning, data with continuous-valued features are usually assumed not only to lie in $\mathbb{R}^n$, but on the Riemannian manifold $\mathbb{R}^n$ equipped with the Euclidean metric tensor, which induces the Euclidean distance measure.  In this context, the Killing vectors form a proper subspace of the generators of the affine transformations. 
However, the Euclidean metric tensor is not the only one used in machine learning: e.g., a different metric tensor may be learned by means of a pullback function induced from a dimensionality reduction algorithm \cite{Sun1}.  In such contexts, the Killing vectors may generate non-affine transformations.
In Section \ref{sec:isometryExp}, we encounter a metric tensor induced by a quadratic feature map, and we restrict our symmetry search space to vector fields which are Killing vectors.

\section{Problem Formulation}
\label{sec:formulation}

Here we precisely formulate the symmetry discovery problem in a way that will inspire our proposed solution.  A \emph{data-generating distribution} $\Dist$ on a differentiable manifold $\Man$ embedded in $\R^n$ is fixed but unknown to us.  A model class $\Hyp$ consisting of vector fields is also fixed.  For instance, $\Hyp$ may be the set of vector fields with coefficient functions that are polynomials with degree at most $d$.
For $X \in \Hyp$ and a smooth function $f:\R^n \to \R$, we denote by 
$\ell_f(X, z) = \ell(X[f](z))$ the \emph{loss} of $X$ at the point $z$ (e.g.,  $\ell_f(X, z) := X[f](z)^2$).  We then define the \emph{risk} of $X$ as
\begin{align}
    R_f(X) := \E_{z \sim \Dist}[\ell_f(X, z)].
\end{align}
Given a dataset $D := \{x_1, ..., x_r\}$ sampled independently from $\Dist$ and a machine learning function $f$, our goal is to identify infinitesimal generators $X \in \Hyp$ that minimize $R_f(X)$.

Formally, we define a \emph{symmetry learner} $L$ to be a function from datasets to elements of $\Hyp$.  We say that $L$ is $(\epsilon, \delta)$-probably approximately correct (PAC) if there is some $r_0(\epsilon, \delta)$ such that, given $r \geq r_0(\epsilon, \delta)$ data samples, for all data-generating distributions $\Dist$, with probability at least $1-\delta$ over sampled datasets, $L$ outputs an $X$ with risk satisfying the following:
\begin{align}
    R_f(X) \leq \epsilon + \inf_{\hat{X} \in \Hyp} R_f(\hat{X}).
\end{align}
Our two goals, then, are symmetry discovery and symmetry enforcement.

    While the estimation of machine learning functions is required prior to symmetry discovery, the novelty contained in the conference paper for which this is an extension dealt primarily with \textit{level set estimation}, which is not the subject of the experiments contained herein. We therefore refer the reader to the original conference paper \cite{Shaw1} for information regarding the estimation of machine learning functions.

\section{Methods}
\label{sec:methods}




There are several components to our methodology. 
First is the estimation of the infinitesimal symmetries of machine learning functions, detailed in Section \ref{sec:method2}. Another component of our methodology is isometry discovery, detailed in Section \ref{sec:IsometryDiscovery}.  In Section~\ref{sec:freedom}, we remark on the ill-posedness of na\"ive formulations
of the symmetry discovery problem.  In Section~\ref{sec:robust}, we give our results analyzing the robustness of symmetry discovery to error in estimation of the machine learning function.
We describe how we evaluate our methods by comparing discovered symmetries to ground-truth symmetries in Section~\ref{sec:GTsym}.
In Section~\ref{sec:enforcement}, we present two techniques for enforcement of discovered symmetries in downstream machine learning models.

\subsection{Estimating the Infinitesimal Generators}
\label{sec:method2}

We first describe the intuition behind our method.
With scalar-valued machine learning functions $f_i$, we construct a vector-valued function $F$ whose components are defined by the functions $f_i$. We can then obtain vector fields which annihilate the components of $F$ by calculating nullspace vectors of $J$, the Jacobian matrix of $F$. 
More vector fields may annihilate the functions than can be identified by the nullspace of $J$, since some vector fields, though linearly independent over the field of real numbers, are not linearly independent over the ring of $\mathcal{C}^{\infty}$ functions. However, our method is inspired by the notion of looking to the nullspace of $J$.
First, we construct a model class of vector fields whose components are norm-constrained linear combinations of pre-determined features, as in a general additive model, though we primarily use polynomial features herein.  The advantage of the choice of polynomial features is that this results in a nested sequence of model classes (indexed by the maximum polynomial feature degree) that is dense in the set of smooth vector fields. The coefficients of the  linear combination of a given model reside in a matrix $W$, with the columns of $W$ corresponding to the coefficients for a single vector field. We then estimate $W$ by solving the following minimization problem, which is constrained by the coefficient norm bound defining the model class: 
\begin{equation}{\label{estVF}}
    W := \argmin_{W_*} \ell(M W_*), 
\end{equation}
where $M = M(\mathcal{J},B)$ is the \textit{Extended Feature Matrix} computed using the array $\mathcal{J}$ of Jacobian matrices at each point, $B$ is the feature matrix, and $\ell(\cdot)$ is the sum of losses over all columns of $W_*$. 
The matrix $M$ is computed via Algorithm \ref{alg:alg1}. 


\begin{algorithm}
\caption{Constructing the Extended Feature Matrix $M$}\label{alg:alg1}
\begin{algorithmic}
\State $m \gets$ number of features
\State $n \gets$ dimension of space
\State $N \gets$ number of points in the dataset
\State $\mathcal{J} \gets Jacobian(F)(x_i)$ \Comment{3-d array of Jacobian matrices at each point}

\For{i in range($N$)}
    \State $row_0 \gets padded(B_i)$ \Comment{m(n-1) 0's are appended to the $i^{th}$ row of $B$}
    \For{j in range(1,m-1)}
    
    $row_{j} \gets roll(row_0,j\cdot m)$ \Comment{Each subsequent row is displaced $m$ entries to the right.}
    \EndFor
    \State $b_i \gets Matrix([row_0, \dots row_{m-1}])$
    \State $mat_{i} \gets \mathcal{J}_i b_i$ \Comment{Multiply the Jacobian of $F$ at $x_i$ by the matrix $b_i$}
    \EndFor
\State $M \gets StackVertical([mat_0, \dots mat_{N-1}])$ \Comment{This matrix has size $nN \times nm$.}

\end{algorithmic}
\end{algorithm}

A solution to Equation (\ref{estVF}) is estimated using constrained optimization of the selected loss function, for which we turn to \textit{McTorch} \cite{mctorch}. Using this PyTorch-compatible library, we can choose from a variety of supported loss functions and (manifold) optimization algorithms. To decide on the number of columns of $W$, we sequentially increase this number until an ``elbow'' in the cost function values is reached. Despite claims to the contrary \cite{hu1}, this approach is amenable to automatic discovery of the number of symmetries, due to existing methods of automatic ``elbow point'' discovery \cite{Onumanyi1}.

However, we find the same unwanted freedom as with level set estimation \cite{Shaw1}, since $X(f)=0$ implies that for any smooth function $h$, $hX(f)=0$. We commonly assume that the components of $X$ are polynomial functions, in which case it is possible that a such a pair $X$ and $hX$ may exist in the search space. There are a few ways of dealing with this issue: primarily, by reducing the search space (to, say, Killing vectors) so that $X$ and $hX$ are not generally both in the search space. 
This issue is further expounded upon in Section \ref{sec:freedom}.

\subsection{Isometry Discovery}
\label{sec:IsometryDiscovery}

As mentioned in Section \ref{sec:freedom}, the fact that a function is $X$-invariant if and only if it is also $fX$-invariant can lead to complications when estimating vector field symmetries. One way  to handle this issue is to restrict the search space of symmetries so that such multiples are not present in the search space. In this section, we outline a technique to restrict the search space to isometries of a distance measure induced by a Riemannian (or pseudo-Riemannian) metric tensor.

Our model assumption is that our data lies on a manifold $\mathcal{M}$ equipped with a metric tensor $g$, and that a machine learning function $f$ on $\mathcal{M}$ is smooth. As stated before, a basis for the vector space of Killing vectors $\{X_i\}_{i=1}^{s}$ can be computed by solving Equation \ref{kvEq}, which is a system of overdetermined, first order, homogeneous partial differential equations known as the Killing equations. Once a basis is identified, we define $X$ to be an arbitrary linear combination of the basis vector fields:
$ 
    X = \sum_{i=1}^{s} c_i X_i.
$ 
We then optimize the real coefficients $\{c_i\}_{i=1}^s$ in the equation $X(f)=0$ subject to the constraint $\sum_{i=1}^s c_i^2=1$. This is a case of Equation \ref{estVF}  where the feature matrix $B_{sr}$ is defined by the components of the Killing vectors, $\mathcal{J}$ is the Jacobian of $f$, and the single column of $W$ is defined by the coefficients $c_i$. To search for multiple Killing vectors which annihilate $f$, one appends additional columns to $W$.

\subsection{A Note about the Freedom Present in Identifying Vector Field Symmetries}
\label{sec:freedom}

Given a vector field $X$ and a vector field $fX$, a function $h$ is $X$-invariant if and only if it is $fX$-invariant. If $h$ is $X$-invariant, the flow of $X$ is a symmetry of $h$, so that the flow of $fX$ is also a symmetry of $h$: this follows from our discussion of vector fields in Section \ref{sec:background}.

    For example, a circle centered at the origin, characterized by $F=0$ where $F(x,y) = x^2+y^2-r^2$ for some real number $r$, exhibits rotational symmetry described by $X=-y\partial_x + x\partial_y$. However, $\frac{1}{y} X (F) = 0$, so that the flow of the vector field $-\partial_x + \frac{x}{y} \partial_y$, where defined, is also a symmetry of the circle.

    This may seem to introduce a theoretical problem requiring greater care when detecting non-affine symmetry. However, our proposed methods of constructing/training models which are invariant with respect to group actions requires only the identification/learning of functions which are invariant with respect to the vector field infinitesimal generators. Thus, both $X$ and $fX$ are valid answers for ``ground truth'' symmetries, since a function $h$ is $X$-invariant if and only if it is $fX$-invariant.
    

    This non-uniqueness presents a challenge when estimating suitable vector fields. When estimating the symmetries via constrained regression, $X$ and $fX$ may or may not both be present in the search space. As mentioned in Section \ref{sec:methods}, one potential workaround is to do symmetry estimation  symbolically. Consider the example of $F=0$ given previously in this section. The Jacobian matrix of $F$ is given as
    $ 
        J = \begin{bmatrix}
            2x & 2y & 2z
        \end{bmatrix},
    $ 
    and a basis for the nullspace of $J$ (using functions as scalars) is $\{ [-y,x,0]^T, [0,-z,y]^T \}$,
    which vectors correspond to tangent vectors $-y\partial_x + x\partial_y$ and $-z\partial_y + y\partial_z$, respectively. The flows of these vector fields correspond with rotations (about the origin) in the $(x,y)$ and $(y,z)$ planes, respectively. However, only two vector fields can be recovered using this method, which in this case has neglected another symmetry, namely $-z\partial_x + x\partial_z$, corresponding to rotations in the $(x,z)$ plane. In fact, each of the three vector fields corresponding to rotations about a coordinate axis are isometries of the sphere described by $F=0$. We should note, however, that the generator for rotations in the $(x,z)$ plane can be expressed as a linear combination of the other rotations using functions as scalars:
    $ 
        \dfrac{z}{y} \left( -y\partial_x + x\partial_y \right) + \dfrac{x}{y} \left( -z\partial_y + y\partial_z \right) = -z\partial_x + x\partial_z.
    $ 
    Thus, it is plausible that applications exist in which a symbolic approach to continuous symmetry discovery may be sufficient. Our approach does not rely on symbolic software, however.

    The issue of unwanted freedom can be addressed by reducing the search space of vector fields to linear combinations of geometrically significant vector fields. In Section \ref{sec:IsometryDiscovery}, we discuss the restriction of the search space to Killing vectors.
        
    Another point which may help to address concerns about our handling of the difficult nature of non-affine symmetry detection is in the relationship between the flows of $X$ and $fX$ generally. The trace of the flow of $X$ through the point $p$ is characterized by the level set $h_i=c_i$, where $\{h_i\}$ is a complete set of $X$-invariant functions. Since a function is $X$-invariant if and only if it is $fX$-invariant, a complete set of invariant functions for $fX$ can be taken to be $\{h_i\}$ without loss of generality. Thus, the level set $h_i=c_i$ also characterizes the trace of the flow of $fX$ through the point $p$. This argument assumes that the flows of $X$ and $fX$, as well as the vector fields themselves, are well-defined in an open neighborhood of the point $p$.
    
    In fact, in our example above with $F = x^2+y^2-r^2$, the flow of $-\partial_x + \frac{x}{y} \partial_y$, assuming $y>0$, is given as
    $ 
        \Phi(t,(x,y)) = \left(x-t, \sqrt{y^2+2tx-t^2} \right),
    $ 
    which trace is (part of) a circle with flow parameter $-x$. The trace of this flow through a point is equivalent to the trace of $X$ through the same point.

\subsection{Robustness analysis}
\label{sec:robust}

Here we present a theorem justifying the estimation of infinitesimal generators of functions that are themselves approximations of an unknown ground truth function (e.g., machine learning functions learned from data).  In the theorem, $f$ plays the role of a ground truth function whose generators we wish to estimate, and $\hat{f}$ plays the role of an approximation of $f$.

\begin{theorem}
    \label{thm:robustness-theorem}

    Let $f:\R^n \to \R$ and $\hat{f}:\R^n \to \R$ be $C_{\infty}$ functions.  Suppose that
    $\|f - \hat{f}\|_{\infty} \leq \epsilon$.  Suppose that $\hat{X}$ is a vector field satisfying
    $\hat{X}(\hat{f}) = 0$ and satisfies the following Lipschitz property: for any smooth functions $g_1, g_2$, 
    \begin{align}
        \|\hat{X}(g_1) - \hat{X}(g_2)\|_{\infty} \leq \lambda \cdot \|g_1 - g_2\|_{\infty}.
    \end{align}
    Let $\hat{\Psi}(t, z)$ be the flow associated with $\hat{X}$.
    Then 
    \begin{align}
        \|\hat{X}(f)\|_{\infty}
        \leq \lambda \epsilon.
        \label{expr:almost-annihilator}
    \end{align}
    As a consequence, we have, for any $t \geq 0$ and any $z \in \R^n$,
    \begin{align}
        |f(\hat{\Psi}(t, z)) - f(z)|
        \leq \lambda \epsilon t.
        \label{expr:almost-invariance}
    \end{align}
    
\end{theorem}
\begin{proof}
    To prove (\ref{expr:almost-annihilator}), we write
    $ 
        \hat{X}(f)
        = \hat{X}(f - \hat{f} + \hat{f})
        = \hat{X}(f - \hat{f}) + \hat{X}(\hat{f})
        = \hat{X}(f - \hat{f}),
    $ 
    where the second equality is by linearity of $\hat{X}$, and the third is by the fact that $\hat{X}(\hat{f})$ annihilates $\hat{f}$.
    Taking norms on both sides and using linearity of $\hat{X}$ and then the Lipschitz property, we have
    $ 
        \|\hat{X}(f)\|_{\infty}
        \leq \lambda \epsilon. 
    $ 

    To prove (\ref{expr:almost-invariance}), we have
    \begin{equation}
        |f(\hat{\Psi}(t, z)) - f(z)|
        = |f(\hat{\Psi}(t, z)) - f(\hat{\Psi}(0, z))|
        \label{expr:supremum-upper-bound}
        = \left|\int_{0}^t \frac{\partial f(\hat{\Psi}(\tau, z))}{\partial \tau} ~d\tau\right|
        \leq t \cdot \sup_{\tau \in [0, t]}\left| \frac{\partial f(\hat{\Psi}(\tau, z))}{\partial \tau}\right|.
    \end{equation}
    Thus, it remains to upper bound $\left| \frac{\partial f(\hat{\Psi}(\tau, z))}{\partial \tau}\right|$ for all $\tau$.  By the chain rule,
    \begin{align}
        \frac{\partial f(\hat{\Psi}(\tau, z))}{\partial \tau}
        = \nabla f(\hat{\Psi}(\tau, z)) \cdot \frac{\partial}{\partial \tau} \hat{\Psi}(\tau, z)
        = \nabla f(\hat{\Psi}(\tau, z)) \cdot \hat{X}(\hat{\Psi}(\tau, z))
        = \hat{X}(f)(\hat{\Psi}(\tau, z)).
    \end{align}
    The second equality is by definition of the flow associated with a vector field.  The third equality is by definition of the vector field applied to functions.  Applying Equation (\ref{expr:almost-annihilator}), we finally get
    $ 
        \left| \frac{\partial f(\hat{\Psi}(\tau, z))}{\partial \tau}\right|
        \leq \lambda \epsilon.
    $ 
    Plugging this into Equation (\ref{expr:supremum-upper-bound}) completes the proof.
\end{proof}

Theorem~\ref{thm:robustness-theorem} says that a vector field that annihilates an approximation of the ground truth function $f$ has a flow that \emph{approximately} serves as an invariance for $f$, where the degree of approximation depends on the smoothness of the discovered vector field and the parameter $t$ of the flow.
Lipschitzness of the vector field $\hat{X}$ can be achieved by constraining it to have, e.g., polynomial coefficients with bounded degrees and monomial coefficients.

\subsection{Comparison of Discovered Symmetries to Ground-Truth Symmetries}
\label{sec:GTsym}

To evaluate symmetry discovery methods, it is important to quantify the ability of a given method to recover specific ground truth symmetries in synthetic experiments.  This is a distinct problem from evaluating the extent to which the discovered symmetry is approximately an infinitesimal generator of an invariant flow, and thus evaluation via the empirical loss is not sufficient. For a ground truth vector field $X$ and an estimated vector field $\hat{X} = \sum_{i=1}^N \hat{f}_i \partial_{x^i}$, we define the angle between them as
\begin{equation}{\label{similarity}}
    \cos\left(\theta(X,\hat{X})\right) = \dfrac{1}{\int_{\Omega}d\mathcal{M}}\mathbb{E}\left[ \dfrac{|\langle X,\hat{X} \rangle_g|}{||X||_{g} \cdot ||\hat{X}||_g} \right],
\end{equation}
where $\langle X,\hat{X} \rangle_g = \sum_{i,j} f_i \hat{f}_j g_{ij}$, $||X||_g = \sqrt{\langle X,X \rangle_g}$, and where 
$ 
    \mathbb{E}\left[ u(\mathbf{x}) \right] = \int_{\Omega} u(\mathbf{x}) d\mathcal{M},
$ 
with the region $\Omega$ being defined by the range of a given dataset. Ordinarily, this is  the full range of the dataset. This formula is a generalization of the formula given in \cite{Shaw1} in the case where the manifold and/or metric is not assumed to be Euclidean.

The notion of a ``cosine similarity'' between vector fields also induces a method by which the extent to which a particular smooth function is invariant can be quantified in relation to other functions. For a fixed vector field $X$, a function $f$ is $X$-invariant if and only if $X$ is orthogonal to the gradient of $f$, which vector field we denote $X_f$. Thus, the extent to which $f$ is $X$-invariant can be quantified in terms of the cosine of the angle between $X$ and $X_f$, given in Equation (\ref{similarity}): the closer to 0 this value is, the more $X$-invariant $f$ is.

\subsection{Symmetry Enforcement}
\label{sec:enforcement}
Here, we detail two approaches to enforcement of discovered symmetries in downstream machine learning models.
In the first, we show that we can construct a polynomial basis for the space of features invariant to 
a given vector field $X$.  By taking polynomial degrees to be sufficiently large, we can guarantee that sufficiently smooth machine learning models trained on these features are arbitrarily close to invariant to $X$.
The second method induces approximate invariance in a new model by regularization of the training objective based on $X$.

\subsubsection{Construction of a basis for the invariant feature space}
\label{sec:invariant-basis-construction}

In this section, we assume that we have already estimated a vector field $X$ that annihilates the target function $f$.  Our next goal is to construct a set of feature functions that are (i) \emph{simple}, (ii)  by design are also annihilated by $X$ and thus share the same invariance as $f$, and (iii) can be combined to approximate a large family of ``well behaved'' functions annihilated by $X$, at least on some domain. Concretely, we seek the following properties for our invariant feature set:

\begin{enumerate}
    \item 
        \textbf{Soundness: } No machine learning model that is \emph{not} $X$-invariant can be constructed from the invariant feature set.
    \item 
        \textbf{Completeness: } ``Every'' $X$-invariant functional can be approximated by some machine learning model built on the invariant features.
\end{enumerate}
We do not claim to construct an invariant feature set which satisfies these properties \textit{exactly}: rather, we seek to \textit{approximately} satisfy these properties.

We note that $X$ is a linear differential operator and that the space
$C^{\infty}(\Omega)$ (for an arbitrary compact set $\Omega \subseteq \R^d$) can be endowed with a norm-inducing inner product:
$ 
    \left\langle v, w \right\rangle := \int_{\Omega} v(z)w(z)~dz.
$ 
With this inner product and norm in mind, we will produce a sequence of bases $B_k := \{f_{1,k}, ..., f_{k,k}\}$ with spans $H_k$ and $C$-approximate spans $H_{k,C}$ 
satisfying the following:
\begin{definition}[Completeness and soundness]
    \label{def:completeness-soundness}

    We say that a sequence $\{B_k\}_{k=1}^\infty$ is \emph{complete} for $X$ if for all $F \in \ker(X)$ and $\epsilon > 0$, there exists
    a sequence index $k_0$ such that for all $k \geq k_0$, 
    $ 
        \min_{\hat{F} \in H_k} \| F - \hat{F}\|_2 \leq \epsilon.
    $ 

    We say that $\{B_k\}_{k\geq 1}$ is \emph{sound} for $X$ if for all error tolerances $\epsilon > 0$ and coefficient bounds $C > 0$, there exists a sequence index $k_0$ such that for all $k \geq k_0$ and  $F_k \in H_{k,C}$,
    $ 
        \| X F_n\|_\infty \leq \epsilon.
    $ 
\end{definition}

\paragraph{Least squares approach to basis construction}

Ultimately, we are trying to solve a first-order linear PDE.  We thus adopt the intuition from linear least squares.
We fix a basis $\hat{B} := \{p_j\}_{j \geq 1}$ for the space of $C^{\infty}$ functions.  For example, this could be a basis of orthonormal polynomials (but this is not needed for our construction).  We write also $\hat{B}_k := \{p_j\}_{j=1}^k$ and $\hat{H}_k := \text{span}\{\hat{B}_k\}$.  We also write $\text{span}_{L_1,B} \{S\}$ for the set of linear combinations of elements of any set $S$ with coefficients $L_1$-bounded by $B > 0$.
We would like to solve $Xf = 0$ for $f \in \hat{H}_k$, but in general only the trivial solution exists.  

We perform the following procedure, which returns for any given $k \in \mathbb{N}$ a set of approximate basis functions
$\{f_j\}_{j=1}^k$ for $\ker(X) \cap \hat{H}_k$:
\begin{enumerate}
    \item 
        Form the matrix
        $ 
            Z := [Xp_1, ..., Xp_k  ].
        $ 
    \item 
        We look at the smallest singular values of the matrix $Z^\dagger Z$.  These correspond to eigenvectors that are coefficient vectors in $\ker(Z)$.  Note that we are viewing $Z$ as a linear operator from $\R^{k}\to C^{\infty}$.
        Then $\ker(Z)$ is the set of coefficient vectors generating linear combinations of elements of $\hat{B}$ that lie in
        $\ker(X)$.

        Specifically, let the singular values of $Z^\dagger Z$ be $\sigma_1\geq \cdots \geq \sigma_k$, and let $v_1, ..., v_k \in \R^k$ denote the corresponding singular vectors (eigenvectors).  We denote
        by $k_0 := \min\{ j \in [k] ~|~ \sigma_j \leq \epsilon\}$.  Here, we recall that $\epsilon > 0$ is an error tolerance parameter.

        Define the functions $\{f_j\}_{j=1}^{k-k_0-1}$ as follows:
        $ 
            f_j := \sum_{\ell=1}^k v_{k_0+j-1,\ell} \cdot p_\ell.
        $ 
\end{enumerate}

Having established a set of functions $\{f_j\}_{j=1}^{k-k_0-1}$, our next task is to establish completeness and soundness for this set, which we do in the Sections \ref{sec:completeness} and \ref{sec:soundness}. 

\paragraph{Establishing completeness}
\label{sec:completeness}
    To establish completeness, we need to figure out a natural restriction of the class of invariant $C^{\infty}$ functions that can be uniformly approximated by the restricted span of $\{f_j\}_{j=1}^k$ by taking $k$ large enough.
    We consider the following \emph{approximation target class}.
    \begin{definition}[Approximation target class]
        \label{def:approximation-target-class}
        For any $\gamma > 0, k \in \mathbb{N}, B \in (0,\infty]$,
        we denote by $F_{\gamma,k,B}$ the set of functions $g \in \ker(X)$ satisfying the following:
        $ 
            \|g - \text{span}_{L_1,B} \{v_1, ..., v_k\}\|_{2} \leq \gamma.
        $ 
        We call $F_{\gamma,k,B}$ the approximation target class.
    \end{definition}

    The significance of the approximation target class is that every element of $\ker(X)$ lies in some approximation target class for large enough $k$ and $B$.  This is because $\{p_j\}_{j=1}^\infty$ is a basis for the space of $C^{\infty}$
    functions on $\Omega$.

    \begin{lemma}[Small $Xg$ in $\hat{H}_k$ implies closeness to $\text{span}\{f_1, ..., f_k\}$]
        \label{lemma:small-Xg-implies-close-to-basis}
        Let $\gamma > 0$.
        Suppose that a function $g \in \hat{H}_k$ satisfies
        $ 
            \|Xg\|_2 \leq \gamma.
        $ 
        Then 
        $ 
            \|g - \text{span}\{f_1, ..., f_{k-k_0-1}\}\|_2 
            \leq \frac{\gamma}{\sqrt{\sigma_{k_0-1}}}
            \leq \frac{\gamma}{\sqrt{\epsilon}}.
        $ 
    \end{lemma}
    \begin{proof}
        We will use the fact that $\{v_1, ..., v_k\}$ is an orthonormal basis.
        We write
        \begin{align*}
            \gamma^2
            &\geq  \|Xg\|_2 
            \geq \|Xg - X\text{span}\{f_1, ..., f_{k-k_0-1}\}\|_2^2 
            = \sum_{j=1}^{k_0-1} c_j^2 \cdot \|v_{j}\|_2^2 \sigma_j 
            \geq \sigma_{k_0-1} \cdot \sum_{j=1}^{k_0-1} c_j^2 \|v_j\|_2^2 \\
            &\geq \epsilon \|g - \text{span}\{f_1, ..., f_{k-k_0-1}\}\|_2^2.
        \end{align*}
        The first inequality is by the hypothesis of the theorem.  The second inequality
        is using orthonormality of $\{v_1, ..., v_k\}$.  Specifically, $\{v_{k_0}, ..., v_k\}$
        is orthogonal to $\{v_1, ..., v_{k_0-1}\}$.  The third inequality is by the fact that
        $\sigma_i \geq \sigma_j$ whenever $i \leq j$.  Finally, the last inequality is
        by definition of $k_0$.
        This completes the proof.
    \end{proof}

    
    \begin{theorem}[Approximating the approximation target class via approximate kernels]
        \label{thm:completeness}
        Let $\gamma > 0$ be arbitrary.  For any $k \in \mathbb{N}$ and $B > 0$, for any $g \in F_{\gamma,k,B}$, 
        we have
        \begin{equation*}
            \|g - \text{span}_{L_1,B}\{f_1, ..., f_{k-k_0-1}\}\|_2
            \leq \gamma + \frac{\lambda_X\gamma}{\sqrt{\epsilon}}.
            \label{expr:completeness-bound}
        \end{equation*}
    \end{theorem}
    \begin{proof}
        By definition of $F_{\gamma,k,B}$, there exists some $\hat{g} \in \text{span}_{L_1,B}\{v_1, ..., v_k\}$
        such that
        $ 
            \|g - \hat{g}\|_2 
            \leq \gamma.
        $ 
        %
        Again from the definition of $F_{\gamma,k,B}$, we have that $Xg = 0$.  This implies 
        that
        $ 
            \|X\hat{g}\|_2
            = \|X\hat{g} - Xg\|_2
            \leq \lambda_X \|\hat{g} - g\|_2
            \leq \lambda_X \cdot \gamma.
        $ 
        Here, we have used the Lipschitz property of $X$, with $L_2\to L_2$ Lipschitz constant $\lambda_X$.
        
        From Lemma~\ref{lemma:small-Xg-implies-close-to-basis}, $\hat{g}$ is close to $\text{span}\{f_1, ..., f_{k-k_0-1}\}$:
        $ 
            \|\hat{g} - \text{span}\{f_1, ..., f_{k-k_0-1}\} \|_2 \leq \frac{\lambda_X \gamma}{\sqrt{\epsilon}}.
        $ 
        Since $\hat{g}$ has coefficients for $\{v_1, ..., v_k\}$ that are bounded by $B$ in the $L_1$ norm,
        this can be strengthened to
        $ 
            \|\hat{g} - \text{span}_{L_1,B}\{f_1, ..., f_{k-k_0-1}\} \|_2 \leq \frac{\lambda_X \gamma}{\sqrt{\epsilon}}.
        $ 

        Putting everything together using the triangle inequality,
        we have that
        \begin{align*}
            \|g - \text{span}_{L_1,B}\{f_1, ..., f_{k-k_0-1}\}\|_2
            \leq \|g - \hat{g}\|_2 + \|\hat{g} - \text{span}_{L_1,B}\{f_1, ..., f_{k-k_0-1}\} \|_2  
            \leq \gamma + \frac{\lambda_X \gamma}{\sqrt{\epsilon}}.
        \end{align*}
    \end{proof}
    Plugging in $\gamma = \epsilon$, we get that (\ref{expr:completeness-bound}) converges to $0$ as $\epsilon \to 0$.

\paragraph{Establishing soundness}
\label{sec:soundness}
    To establish soundness, we restrict the span of the $\{f_j\}$ to have $B$-bounded coefficients (in the $L_1$ norm).  This will allow
    us to upper bound $\|Xg\|_{2}$ for arbitrary $g$ in this set.  This upper bound implies an upper bound on the $L_1$ norm that grows with the volume of the domain $\Omega$.  This, in turn, implies that no non-negligible subset can have $Xg$ too large in $L_{\infty}$ norm.

    \begin{theorem}[Soundness]
        \label{thm:soundness}
        Let $B > 0$.  We have that for any $g \in \text{span}_{L_1,B}\{f_1, ..., f_{k-k_0-1}\}$,
        $
            \|Xg\|_2 \leq \sqrt{B\epsilon}.
        $ 
        This implies
        $ 
            \|Xg\|_1 
            \leq \text{Vol}(\Omega) \cdot \|Xg\|_2 \leq \text{Vol}(\Omega)\sqrt{B\epsilon}.
        $ 
    \end{theorem}
    \begin{proof}
        The first inequality 
        follows immediately from the choice of $k_0$ and the
        $L_1$ bound $B$.
        The second inequality follows from elementary properties of norms.
    \end{proof}
    We note that soundness has no dependence on $k$, because of the choice of $k_0$.

\subsubsection{Enforcing symmetries via regularization}

We now turn to the topic of enforcing continuous symmetries in a smooth machine learning model during training using regularization based on discovered infinitesimal generators $\{X_k\}_{k=1}^s$.

Suppose that we seek to learn a function $f$ mapping data instances $x_i$ to targets $y_i$, where $x_i \in \mathbb{R}^d$ and $y_i \in \mathbb{R}^m$. Suppose also that the function $f$ is estimated by means of the minimization of $L(f(\mathbf{x}),\mathbf{y})$ for a smooth loss function $L$.

Now suppose that we also  desire that the function $f$ learns to exhibit continuous symmetry with respect to the independent infinitesimal generators $\{X_k\}_{k=1}^s$. That is, for each component $f_j$ of $f$,
$ 
    X_k(f_j)=0,
$ 
for $1 \leq k \leq s$ and $1 \leq j \leq m$. Our method of symmetry enforcement in training is to minimize the following loss function:
\begin{equation}{\label{EnforcementEqn}}
    (1-\lambda(t)) L(f(\mathbf{x}),\mathbf{y}) + \lambda(t) \Tilde{L}(\mathbf{X}(f)(\mathbf{x}),\mathcal{O}),
\end{equation}
where $\mathbf{X}(f)(\mathbf{x}) = (X_k(f_j))(x_i)$, $\mathcal{O}$ is an array with zero components of the same shape as $\mathbf{X}(f)(\mathbf{x})$, $\Tilde{L}$ is a smooth loss function, and $\lambda(t) \in [0,1]$ is a (possibly time-dependent) symmetry regularization parameter.

The loss function given in Eq. (\ref{EnforcementEqn}) resembles that of a PINN, where our $\Tilde{L}$ term corresponds to differential constraints imposed by the governing equations for a given physical system \cite{Raissi1}. 
Additionally, our symmetry enforcement term is a differential constraint, so that our method of symmetry enforcement appears to be a special case of a PINN in which the differential operators are linear, first-order, and homogeneous. Despite the mathematical similarity of enforcing continuous symmetries using vector fields and the enforcement of differential equations governing a physical system, we note that PINNs do not erase the need for symmetry enforcement using vector fields. This is due simply to the fact that the desired symmetry characteristics of a machine learning model are not always readily expressible in terms of a physical system. We note, however, that a machine learning task subject to physical constraints may be amenable to vector field regularization--in particular, a \textit{Hamiltonian vector field} can often be readily given for a Hamiltonian defined on a symplectic manifold.


\subsection{Limitations}
\label{sec:limitations}

Our methods are state of the art for continuous symmetry detection and enforcement, both for computational reasons and due to the fact that our methods can readily handle continuous symmetries that are not affine transformations without the need to discretize continuous transformations. However, there are limitations to our methods. First, as we have mentioned previously, is the issue which arises when vector fields $X$ and $fX$ are both present in the search space. We have suggested that a symbolic approach may resolve this, though non-symbolic approaches such as ours must grapple with this issue. Our approach of restricting the model class to low-degree polynomials partially resolves this issue. Another promising resolution is to restrict the search space of vector fields to particular types of symmetries, such as isometries, as described in Section \ref{sec:IsometryDiscovery}. Further work is required to fully resolve this.

Another limitation is that our method requires an assumption about the coefficient functions of the vector fields. Herein, we primarily use polynomial functions, and although polynomials are universal approximators \cite{PolyApprox}, the number of polynomial terms for a given coefficient function can increase rapidly, particularly when the number of dimensions is high. We believe this limitation is best resolved, as with the previous limitation, by restricting the search space to particular types of symmetries. 

Though our method is comparatively efficient, the reliance on the elbow curve to determine both the number of discovered symmetries and, when level set estimation is used, the number of coefficients in the level set, is a limitation. Unlike the previous limitations, this limitation is shared by methods for other common problems, appearing in analogous form as the dimension hyperparameter in manifold learning methods, or as the number of clusters in $k$-means clustering. It is possible that our method could be improved by another means of selecting this number. 

We also mention the limitations inherited via the choice of optimizing Equation (\ref{estVF}) 
Cost function selection, optimization algorithm selection, hyperparameter selection, together with the potential for poor parameter initialization all constitute limitations to our method. However, the optimization task is much more straight-forward than training many other symmetry detection models,  as most other methods require fine-tuning neural networks.


\section{Experiments}
\label{sec:experiments}

We offer three experiments to provide for an effective demonstration of our methods. 
First, we learn the infinitesimal symmetries of a machine learning model given as a neural network. While this has been previously suggested as possible in principle (see \cite{Shaw1}), it was not explicitly accomplished. The second experiment enforces non-affine symmetry in model training, obtaining high test accuracy without the need of data augmentation. 
Our final experiment discovers infinitesimal isometries of a classification function, where the distance measure in the original feature space is induced by a quadratic feature map.

\subsection{Symmetry Discovery and Enforcement for Neural Networks}
\label{sec:nnExp}

In our first experiment, we demonstrate how a feedforward neural network can learn a rotationally symmetric function from 2D input data,how its estimated derivatives can be used to detect the underlying symmetry, and how that learned symmetry can be used to further enforce model symmetry for better generalization. For the training set, we generate $N=1000$ points $(x,y)$ uniformly from a unit disk by first sampling 
$ 
    r = \sqrt{U}, \quad \theta \sim \mathrm{Uniform}(0,\pi), \quad U \sim \mathrm{Uniform}(0,4),
$ 
and then converting to Cartesian coordinates by means of $x = r\cos(\theta)$ and $y = r\sin(\theta)$. The test set is sampled similarly, but with $N=3000$ and $\theta \sim \mathrm{Uniform}(\pi,2\pi)$.

Our regression target is $f(x,y)=x^2+y^2$, which is invariant under planar rotations about the origin. A ground-truth infinitesimal generator for this rotational symmetry is 
$ 
    X = -y\partial_x + x\partial_y.
$ 
We define a multi-layer perceptron $F:\mathbb{R}^2\to\mathbb{R}$ with two hidden layers of size 64 and train it for 5000 epochs using the Adam optimizer at a learning rate of $0.01$. The final mean-square error (MSE) on the training set is on the order of $10^{-5}$, while the test MSE is approximately $0.8181$. 

To uncover the underlying rotational symmetry from the network’s derivatives, we apply Equation (\ref{estVF}) under the assumption that the coefficients of the estimated vector field $\hat{X}$ are affine. The gradient of $F$ is approximated using 
autograd. Using the L1 Loss, the (Riemannian) Adagrad optimizer with learning rate $0.1$, and training for $1000$ epochs, we recover the following estimate of the symmetry of $F$:
$ 
    \hat{X} = (0.0012x + 0.7144y - 0.0135) \partial_x + (-0.6996x - 0.0059y + 0.0037)\partial_y,
$ 

We now use the learned symmetry to retrain the model with symmetry regularization. We train a new model as before, only with $4500$ epochs instead of $5000$. Then, for the remaining $500$ epochs, we regularize with symmetry using equation \ref{EnforcementEqn} with $\lambda=0.01$, $\mathcal{L}_2$ being the MSE loss, and a new learning rate of $0.001$. This results in a test MSE (for labels) of $0.1148$, a fraction of the test MSE of $0.8181$ obtained without symmetry enforcement.

\subsection{Non-affine Symmetry Enforcement}
\label{sec:naExp}

Our second experiment compares ordinary polynomial regression with polynomial regression using our symmetry enforcement method. First, we generate 2000 datapoints $(x_i,y_i)$ in the plane in a ``thin strip'' configuration: each $x_i$ is generated from $\text{Uniform}(-4,4)$, and $y_i=0+\epsilon_i$, where $\epsilon_i \in (-0.1,0.1)$ (distributed uniformly) is a random noise component. For each point, we assign regression targets according to the function
\begin{equation}{\label{exp2f}}
    f = 2x^4-2x^2y^2+y^4.
\end{equation}
A ground truth symmetry of this function is given as
\begin{equation}{\label{exp2X}}
    X = (x^2y -y^3) \partial_x + (2x^3-xy^2) \partial_y.
\end{equation}
To apply ``ordinary'' polynomial regression, we seek to minimize the MSE between the output of the function $f$ given in Equation (\ref{exp2f}) and the output of an arbitrary degree four polynomial, iteratively updating the coefficients. We use gradient descent, a learning rate of $10^{-4}$, and $100,000$ epochs.
This results in a training MSE of $0.0159$ and a learned function given approximately as
\begin{equation}
    \hat{f} = -0.0028 + 0.0023 x  -0.7872 y -0.0022 x^2 + 2.2454 xy + 0.4294 y^2  -0.0001 x^3 + 0.0546 x^2y
\end{equation}
\begin{equation*}
    -0.5132 xy^2 + 0.0426 y^3 +1.9999 x^4 -0.1950 x^3y -0.9596 x^2y^2 + 0.8888 xy^3 + 1.7187 y^4.
\end{equation*}
With $\langle \hat{f}, f \rangle_{\text{train}} = \int_{-0.1}^{0.1}\int_{-4}^{4} \hat{f} f dx dy$, we find that
$ 
    \cos(\theta_{\text{train}}) = \dfrac{\langle \hat{f}, f \rangle_{\text{train}}}{ ||\hat{f}||_{\text{train}} ||f||_{\text{train}}} \approx 0.9999.
$ 
However, to examine the generalizability of the learned function on the region $[-3,3] \times [-3,3]$, we define $\langle \hat{f}, f \rangle_{\text{test}} = \int_{-3}^{3}\int_{-3}^{3} \hat{f} f dx dy$, resulting in 
$ 
    \cos(\theta_{\text{test}}) = \dfrac{\langle \hat{f}, f \rangle_{\text{test}}}{ ||\hat{f}||_{\text{test}} ||f||_{\text{test}}} \approx 0.8847.
$ 
Additionally, a visualization of the predictions for $\hat{f}$ are given in Figure \ref{fig:nonaffine}. To estimate $f$ using symmetry enforcement, we apply Equation (\ref{EnforcementEqn}) with $\lambda(t)=0.5$, with both loss functions being the MSE. Training, as before, for $100,000$ epochs, we obtain a (total) training loss function minimum of $1.4 \cdot 10^{-6}$, with
$ 
    \hat{f}_s \approx 0.0001 + 0.0004 y^2 + 2.0000 x^4 - 2.0000 x^2 y^2 + 0.9594 y^4.
$ 
We find that
$ 
    \cos(\theta_{\text{train}}) \approx 1.0000, 
    \cos(\theta_{\text{test}}) \approx 0.9998,
$ 
and, as Figure \ref{fig:nonaffine} reveals, model predictions whose symmetries appear to be more in line with that of the ground truth function.

\begin{figure}
    \centering
    \includegraphics[width=\linewidth]{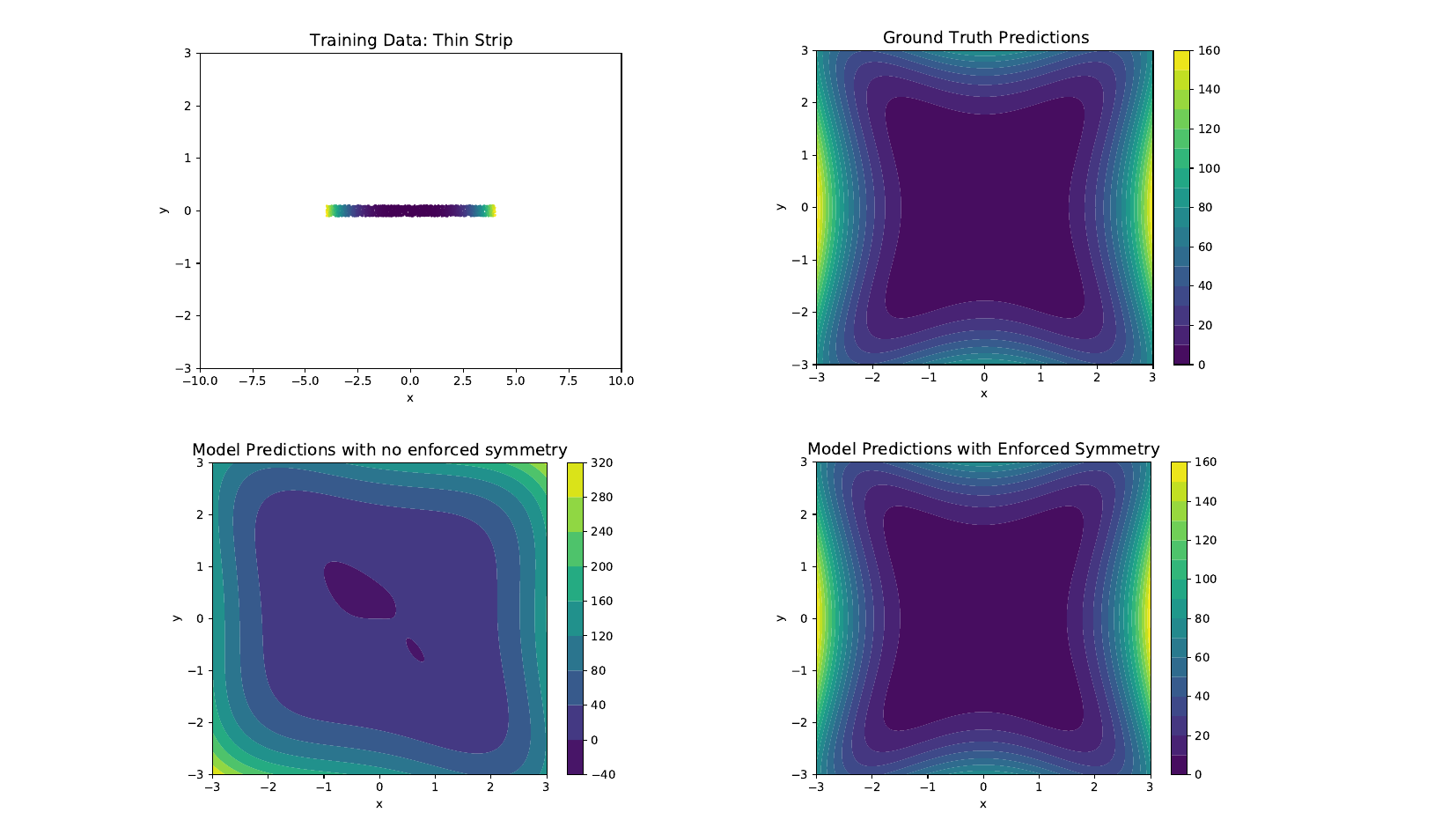}
    \caption{Symmetry enforcement visualizations for the second experiment. Top left: an illustration of the training data and regression labels. Top right: a contour plot of the ground truth regression function. Bottom left: a contour plot of the model trained without symmetry enforcement. Bottom right: a contour plot of a model trained by enforcing symmetry.}
    \label{fig:nonaffine}
\end{figure}

\subsection{Isometry Discovery} 
\label{sec:isometryExp} 
We consider a 3-dimensional classification problem in which points $(x,y,z)$ are assigned to class 1 if $f(x,y,z) > 0$ and to class 0 otherwise, where
$ 
f(x,y,z) = x^2 + \tfrac{y^2}{2} - yz + \tfrac{z^2}{2} - \tfrac{1}{2}.
$ 
 To learn the class separation, we map the original data to a 6-dimensional Euclidean feature space only using the degree-2 terms of a polynomial feature map:
\begin{equation*}
\Phi(x,y,z) \;=\; \bigl[\,x^2,\;\sqrt{2}\,x\,y,\;\sqrt{2}\,x\,z,\;y^2,\;\sqrt{2}\,y\,z,\;z^2\bigr].
\end{equation*}
This mapping results in the following pullback metric on the original 3-d space:
\begin{equation*}
g =
\begin{bmatrix}
4x^2 + 2y^2 + 2z^2 & 2yx & 2zx\\ 
2yx & 2x^2 + 4y^2 + 2z^2 & 2yz\\ 
2zx & 2yz & 2x^2 + 2y^2 + 4z^2
\end{bmatrix}.
\end{equation*}
Next, a logistic regression classifier is trained in the 6-d feature space, producing a decision boundary that, when translated back to the original coordinates, approximately corresponds to the quadratic form $f(x,y,z)=0$. We wish to learn the isometries of the model probabilities, that is, the Killing vectors which annihilate 
$ 
p(x,y,z) \;=\; \frac{1}{\,1 + e^{-\,f(x,y,z)}\,}.
$ 
The metric $g$ admits three linearly independent Killing vectors, the flows of which are 3-d rotations:
\begin{equation*}
X_1 = -\,z\,\partial_y + y\,\partial_z,
\quad
X_2 = -\,z\,\partial_x + x\,\partial_z,
\quad
X_3 = -\,y\,\partial_x + x\,\partial_y.
\end{equation*}
We now construct an arbitrary linear combination
$ 
\hat{X} = aX_1 + bX_2 + cX_3
$ 
and optimize $(a,b,c)$ via Equation~(\eqref{estVF}). We note that a ground truth solution $X$ has coefficients $(0,-1,1)$. Using the L1 Loss function and the (Riemannian) Adagrad optimizer with a learning rate of $0.01$, we obtain the coefficients $(0.0055, 0.7076, -0.7066)$, and the similarity between $X$ and $\hat{X}$ is estimated to be $0.99992$. 

Thus, we have estimated the isometries of a model probability function. Because $\hat{X}$ approximately annihilates $p(x,y,z)$, the probability value of a point $\mathbf{x}$ is approximately invariant under the flow of $\hat{X}$, which is a planar rotation. Because the flow of $\hat{X}$ is an isometry, the distance (induced by the metric $g$) between two points remains unchanged when the two points are simultaneously rotated by an angle $\theta$ according to the rotation defined by the flow of $\hat{X}$.

\section{Conclusion}
While most current state of the art experiments focus, in the context of continuous symmetry detection, on affine symmetries, we have outlined a method to estimate symmetries including and beyond affine transformations using vector fields. Our method takes as input an estimated machine learning function such as a probability distribution, continuous targets for regression, or functions which characterize an embedded submanifold (obtained via level set estimation). Vector fields are constructed as arbitrary linear combinations of a chosen basis of vector fields, and the parameters that are the coefficients of the linear combination are optimized, often using constrained optimization. Herein, we have typically expressed the coefficients of our vector fields in terms of polynomial functions, excepting isometries, in which a linear combination of Killing vectors is formed. Future work includes detecting other types of symmetries, including angle-preserving transformations whose infinitesimal generators are known as conformal Killing vectors \cite{Shaw2}.

When searching for affine symmetries in low dimensions, our method is appealing when compared with many current methods. This is due primarily to the relative ease of implementation. Additionally, when compared with the state of the art, our method offers a computational advantage while at least competing and often outperforming in terms of accuracy, especially for more complex symmetries than affine ones \cite{Shaw1}.

Infinitesimal generators are a step removed from the group action. 
We have proposed building invariant architectures based on the identification of features which are invariant with respect to the discovered vector fields, resulting in an invariant model, since the information fed to the model is approximately invariant with respect to the group action. We also proposed a means of symmetry enforcement in the training of a model though symmetry regularization. Both methods may help to improve model generalization.




\bibliographystyle{siamplain}
\bibliography{main}

\end{document}